\documentclass[12pt, letter]{article}

\usepackage{amsthm}
\usepackage[margin=1in]{geometry}
\usepackage{palatino}
\usepackage{xcolor}
\usepackage[width=\textwidth]{caption}
\usepackage{subcaption}
\usepackage[pdftex]{graphicx}
\usepackage[skip=0pt]{caption}
\usepackage{booktabs}
\usepackage{setspace}
\usepackage{lscape}
\usepackage{bm}
\usepackage[hidelinks]{hyperref}
\hypersetup{
     colorlinks   = false,
     allcolors    = [RGB]{163,31,52}
}
\usepackage{indentfirst}
\usepackage[round]{natbib}
\usepackage[T1]{fontenc}
\usepackage[utf8]{inputenc}
\usepackage{textcomp}
\usepackage{amsmath,amssymb,amsfonts,mathtools}
\usepackage{float}
\usepackage{array}
\usepackage{enumerate}
\usepackage{enumitem}

\renewcommand{\thesection}{\arabic{section}}

\definecolor{mit red}{RGB}{163, 31, 52}

\newtheorem{theorem}{Theorem}

\theoremstyle{definition}
\newtheorem{assumption}{Assumption}
\newtheorem{definition}{Definition}

\DeclareMathOperator{\E}{\mathbb{E}}

\newcommand{\R}{\mathbb{R}}

\newcommand{\ve}{\varepsilon}

\newcommand{\trans}{^{\top}}
\newcommand{\inv}{^{-1}}

\newcommand{\m}{m_0}               % true transition function
\newcommand{\Jz}{J_0}             % true network Jacobian
\newcommand{\alphaij}{\alpha_{ij}}
\newcommand{\psiij}{\psi_{ij}}
\newcommand{\thetaij}{\theta_{ij}}
\newcommand{\hattheta}{\hat{\theta}_{ij}^{\mathrm{deb}}}
\newcommand{\Zt}{Z_t}             % state at time t
\newcommand{\embed}{\mathbf{e}}    % GNN node embedding

\title{%
  \Large Network Recovery from Cascade Data:\\
  A Debiased Jacobian-Based Machine Learning Approach%
  \footnote{Comments welcome.
  Lei Huang (\texttt{leihuang@mit.edu}) is a Ph.D.\ candidate
  at the MIT Sloan School of Management.  All errors are my own.}
}
\author{Lei Huang}

\begin{document}
\date{June 1, 2026}
\maketitle

%% ── Abstract ────────────────────────────────────────────────────────────────
\begin{abstract}
\singlespacing
\noindent
Many important outcomes unfold as dynamic cascades, including product adoption, disease spread, financial distress, and information diffusion. A central challenge is to recover the hidden influence network behind these cascades. Existing methods typically assume a specific diffusion model, and their performance degrades substantially when that assumption is misspecified. We propose \emph{CascadeNet}, a Jacobian-based machine learning framework for network recovery that does not require specifying a diffusion mechanism. The key idea is that the underlying influence structure can be characterized by the Jacobian of the one-step transition function. CascadeNet first constructs a flexible estimator of the transition function, and further applies Neyman-orthogonal debiasing via the Riesz representer, so that the debiased Jacobian is $\sqrt{n}$-consistent and asymptotically normal, enabling formal inference on the network structure. We validate {CascadeNet} in both a simulation exercise and a real-world empirical application. In simulations, where the data-generating process is known, {CascadeNet} achieves the highest network recovery accuracy across nine common data-generating processes. In an empirical application to COVID-19 transmission across Spain's 52 provinces, {CascadeNet} recovers transmission networks that are significantly correlated with the true inter-province mobility network, whereas networks recovered by baseline methods show no significant alignment with the ground truth.

\bigskip
\noindent \emph{\textbf{Keywords}}: network inference, Jacobian, debiased machine learning.

\end{abstract}
\newpage
\doublespacing

%% ════════════════════════════════════════════════════════════════════════════
\section{Introduction}
\label{sec:intro}

Networks shape economic and social outcomes.  Whether a firm is designing a viral marketing campaign, a platform is studying how product reviews propagate across its user base, or a public health agency is tracing disease transmission, the effectiveness of any intervention depends on the structure of the underlying network \citep{kempe2003maximizing, ballester2006who, bramoulle2009identification}.  Yet in practice, the network is rarely observed directly.  Firms do not see who actually persuaded whom to buy; platforms do not see which user influenced another's review; public health agencies do not see which infected individual transmitted the pathogen to which other individual.  What decision-makers can often observe are the cascades: the sequences of adoption, engagement, or infection events that propagate across the unobserved network over time.  The central task is then network recovery: inferring the hidden influence structure from these observed cascades, so that decision-makers can identify the relationships through which influence flows and target their interventions accordingly.

The economic stakes of network recovery are large.  Influence-maximization campaigns, lockdown decisions in epidemiology, contact tracing protocols, and systemic-risk monitoring all depend on ranking edges by their strength: a misranked edge can mean a wasted promotional budget, a missed superspreader, or a financial contagion channel that is overlooked.  Yet the same data poverty that makes network recovery valuable also makes it hard.  An analyst typically observes only a handful of cascades over a moderately sized population, and must infer at least an $N \times N$ influence matrix, where $N$ is the number of nodes or players, or a higher-dimensional object when more complex interactions are assumed, from this thin record.

To tackle this challenge, a substantial literature has developed methods that assume a very specific functional form for the underlying transition process: for instance, the exponential and power-law transmission kernels in \citet{gomezrodriguez2010inferring} and \citet{gomezrodriguez2011uncovering}, the pairwise Markov transition probabilities in \citet{ramezani2024dani}, and the particular GNN architectures in \citet{qiu2018deepinf} and \citet{murphy2021deep}. These functional-form assumptions provide identification power that the data alone cannot, but when the assumed transition process is misspecified, the resulting error can be large. These methods also share a second drawback: they do not provide formal statistical inference on the recovered edges, which makes it hard for practitioners to communicate uncertainty about the recovered structure and to make informed decisions based on it. We propose \emph{CascadeNet}, a Jacobian-based machine learning framework for network recovery that addresses both limitations.

The key insight of {CascadeNet} is that, in a Markovian system, the influence network can be characterized by the Jacobian of the one-step transition function.  This insight allows us to formulate network recovery as a problem of estimating the transition function and then differentiating it.  By replacing parametric likelihood with a flexible learner of the one-step transition function, {CascadeNet} avoids the misspecification risk that plagues classical methods and enables application to a wider range of diffusion mechanisms, outcome types, and covariate structures.  We further apply Neyman-orthogonal debiasing based on the Riesz representer, a technique from the double machine learning literature, to address bias from regularization in the flexible machine learning estimator.  This allows us to recover an unbiased estimate of the network Jacobian using only observable prediction residuals, and to prove $\sqrt{n}$-consistency and asymptotic normality for each debiased Jacobian entry, enabling formal inference on individual edge weights.

To lay out the problem, consider a panel of $N$ agents, the nodes of the network, observed over time.  The state vector $Y_t$ records each agent's outcome at time $t$, such as adoption status, infection status, or case counts, and evolves according to a one-step transition function, $Y_{t+1} = m_0(Y_t, X_t) + \varepsilon_{t+1}$, where $X_t$ collects observed covariates such as treatment assignments, demographics, or policy variables, and $m_0$ is left fully unspecified.  This formulation nests classic diffusion models, including independent cascade, linear threshold, SIS, SIR, and Hawkes, as special cases, and accommodates binary, count, or continuous outcomes.

CascadeNet has two components.  The first is a flexible estimator $\hat{m}$ for the one-step transition function $m_0$, and the same estimator is used iteratively across time steps.  The simplest specification is a linear-index model with a learnable interaction matrix $J$ and a flexible link function (sigmoid for binary adoption, identity or MLP for continuous outcomes), which already strictly nests classical baselines; richer specifications such as graph neural networks with attention-based aggregation can be plugged in when the application demands them.  When the model class is large enough, the misspecification risk is minimal. The second component is the central observation that turns the estimated transition function into a network estimator: the Jacobian of $m_0$, the $N \times N$ matrix $J_0$ with entries $J_0[i,j] = \partial m_i / \partial y_j$, directly encodes the influence network.  In a Markovian system, the $(i,j)$ entry is nonzero if and only if agent $j$ can marginally influence agent $i$, and its magnitude measures the strength of that influence; network recovery therefore reduces to differentiating the estimated transition function, a single operation that works regardless of which parametric diffusion model, if any, generated the data.

However, this flexibility of the transition function estimator comes at a cost.  Common machine learning estimators are regularized to control variance, and regularization introduces bias.  For instance, $\ell_1$ and $\ell_2$ penalties shrink the estimated parameters toward zero, and therefore bias the estimated Jacobian; in a small-sample setting where regularization is most needed, this attenuation can be severe enough to destroy the network signal entirely.  Off-the-shelf machine learning alone is therefore not enough.  To address this bias, we apply Neyman-orthogonal debiasing based on the Riesz representer \citep{chernozhukov2022automatic, hirshberg2021augmented}, a technique from the double machine learning literature \citep{chernozhukov2018double}.  Intuitively, the Riesz representer acts as an integration-by-parts operator: it converts errors in the prediction $\hat{m}$ into corrections for the derivative $\partial \hat{m}_i / \partial y_j$, allowing us to recover an unbiased estimate of the network Jacobian using only the observable prediction residual.  The resulting debiased estimator is $\sqrt{n}$-consistent and asymptotically normal for each individual edge, even though $m_0$ itself is estimated by flexible machine learning.  We illustrate the necessity and effectiveness of this debiasing strategy in a controlled tanh experiment with a closed-form ground truth: the naive plug-in Jacobian achieves a Pearson correlation of only $r = 0.07$ with the true Jacobian, while the Riesz correction restores it to $r = 0.77$, more than a tenfold improvement over the naive plug-in.

We validate CascadeNet on both synthetic and real data.  On synthetic cascades generated from nine common diffusion models from the literature, namely pairwise transmission (IC, LT), continuous-time SI (Exp, PL), epidemic models with recovery/removal (SIS, SIR), aggregate-effect models (Complex, Hawkes), and a flexible nonlinear DGP, CascadeNet achieves the highest recovery accuracy in all nine settings, outperforming established methods by up to 50 percentage points of Pearson correlation on diffusion models that violate the baselines' kernel assumptions.  On real COVID-19 case data from Spain's 52 provinces, CascadeNet is the only method whose estimated edges correlate significantly with true inter-province mobility flows; every classical baseline produces rankings indistinguishable from random. 

CascadeNet makes two methodological contributions.  First, it provides a flexible estimator for the influence network: by replacing parametric likelihood with a flexible learner of the one-step transition function, CascadeNet avoids the misspecification risk that plagues classical methods and enables application to a wider range of diffusion mechanisms, outcome types, and covariate structures.  Second, it provides formal statistical inference: we extend the Riesz representer framework to an $N \times N$ matrix of partial derivatives, a setting not previously addressed in the debiased ML literature, and prove $\sqrt{n}$-consistency and asymptotic normality for each debiased Jacobian entry.  This delivers what classical cascade-inference methods cannot: confidence intervals for individual edge weights, hypothesis tests for whether a particular edge exists, and a principled basis for communicating uncertainty about the recovered network.  

Managerially, CascadeNet equips decision-makers with the ability to map influence relationships from observed cascade data alone, with formal confidence intervals on individual edges, without requiring the analyst to specify how influence propagates.  In practical settings such as a marketing manager planning a seeding campaign, a public-health authority allocating contact-tracing resources, or a regulator monitoring contagion in a financial network, this combination of flexibility and inferential rigor is what allows the recovered network to be used as a serious input to consequential decisions, rather than as a heuristic ranking with unknown reliability.

The remainder of the paper is organized as follows.  Section~\ref{sec:literature} reviews the related literature.  Section~\ref{sec:cascade} presents the CascadeNet framework, including the data-generating process, the estimator, the Jacobian representation of the network, the debiasing procedure, and the asymptotic theory.  Section~\ref{sec:simulation} reports synthetic validation.  Section~\ref{sec:covid} applies the method to COVID-19 data.  Section~\ref{sec:managerial} discusses managerial implications and concludes.

%% ════════════════════════════════════════════════════════════════════════════
\section{Related Literature}
\label{sec:literature}

Our work contributes to several streams of research.  We discuss each in turn and highlight how CascadeNet relates to and extends prior work.

\paragraph{Cascade network inference.}
The problem of inferring networks from cascade data has attracted considerable attention in the network science literature.  The foundational contributions are NetInf \citep{gomezrodriguez2010inferring} and NetRate \citep{gomezrodriguez2011uncovering}, which infer diffusion networks by maximizing pairwise cascade likelihoods under exponential or power-law transmission kernels and rely on submodular greedy selection or convex optimization for tractability.  CONNIE \citep{myers2010network} casts the inference problem as a convex program with a regularized likelihood, providing the first scalable formulation with provable optimality.  More recently, DANI \citep{ramezani2024dani} takes a Markov-transition approach that estimates pairwise transmission probabilities directly from the observed panel and emphasizes preserving topological structure.  A parallel line of work uses deep learning to model cascade dynamics: DeepInf \citep{qiu2018deepinf} predicts user-level adoption with attention-based GNNs, and \citet{murphy2021deep} learn contagion dynamics on complex networks with graph neural networks, demonstrating strong predictive performance on epidemiological data.  Despite their methodological diversity, these methods share two features that motivate our work.  First, each imposes a specific parametric assumption on how influence propagates (an exponential hazard, a power-law kernel, a linear Markov transition, or a particular GNN architecture), and performance is sensitive to whether that assumption matches the true diffusion process.  Second, none provide formal statistical inference on individual edges: the output is a ranked list or a point estimate, but not a confidence interval or a hypothesis test.  Our contribution is to develop a flexible estimator that avoids the parametric kernel assumption, accommodates heterogeneous diffusion mechanisms within a single framework, and provides $\sqrt{n}$-consistent debiasing guarantees that enable formal edge-level inference.

\paragraph{Networks in marketing.}
A rich marketing literature studies how networks shape consumer behavior, peer effects, and the diffusion of new products.  \citet{kempe2003maximizing} formalize the influence-maximization problem of selecting seed nodes to maximize cascade reach, which has become a central decision problem in viral marketing and presupposes knowledge of the influence network.  Empirical work on word-of-mouth and social contagion in marketing \citep{ballester2006who, chu2016quantifying, eckles2016peer} similarly takes the network as known or proxies it with observable ties (e.g., friendships, co-purchase records, geographic proximity).  When the network is not directly observable, marketing researchers have used model-based approaches to infer peer effects from panel data, but these typically assume linear-in-means specifications \citep{manski1993identification, bramoulle2009identification} that may not capture the nonlinear dynamics of cascade adoption.  Our work complements this literature in three ways.  First, we provide a method to infer the influence network directly from observed cascades, without requiring auxiliary network data, thereby expanding the set of settings in which network-based marketing analysis is feasible.  Second, by supplying confidence intervals on individual edges, we enable practitioners to communicate the statistical reliability of recovered seeding targets, a feature that has been missing from existing influence-maximization pipelines.  Third, the flexible form of CascadeNet allows it to capture richer adoption dynamics (threshold effects, saturation, social reinforcement) than linear-in-means models, while still delivering tractable inference.

\paragraph{Debiased machine learning.}
Our inference theory builds on the double/debiased machine learning (DML) framework of \citet{chernozhukov2018double} and the automatic-debiasing approach via the Riesz representer developed by \citet{chernozhukov2022automatic} and \citet{hirshberg2021augmented}.  The general principle is that for a smooth functional of a nuisance function estimated by machine learning, an orthogonal score correction restores the parametric $\sqrt{n}$ rate of convergence even when the nuisance estimator converges only at $n^{-1/4}$.  This framework has been applied successfully to a wide range of settings, including heterogeneous treatment effects, policy learning, structural parameter estimation, and causal effects with high-dimensional controls \citep{farrell2021deep, farrell2020individual, ye2025deep}.  Two features of our setting are new relative to existing applications.  First, the target of inference is not a low-dimensional structural parameter but an $N \times N$ matrix of partial derivatives, the network Jacobian, with an entry-specific orthogonal score for each pair $(i,j)$.  Showing that the Riesz representer admits a closed form in the linear-index case, and a tractable loss-based approximation in the GNN case, is one of the technical contributions of this paper.  Second, the underlying data are panel cascade trajectories rather than i.i.d.\ cross-sectional observations, which requires care in setting up the cross-fitting scheme and in justifying the asymptotic argument under repeated sampling of trajectories.  More broadly, our paper extends the DML toolkit from causal effect estimation, where it has dominated, to network recovery, where we believe it has equally large practical implications.

\paragraph{Network econometrics.}
A substantial econometrics literature studies identification and estimation of causal effects in network settings, typically taking the network as observed and analyzing how peer interactions, spillovers, and equilibrium effects shape outcomes.  \citet{manski1993identification} formulates the reflection problem that limits identification of endogenous peer effects, and \citet{bramoulle2009identification} establish identification conditions when the network is known.  More recent work studies causal inference under network interference \citep{eckles2017design, leung2022causal} and equilibrium effects of policy interventions \citep{wager2021experimenting, jiang2022estimating}, again under the assumption that the relevant network structure is observed by the analyst.  Our paper sits upstream of this literature: we provide a method to recover the network itself from cascade data, with formal inference on individual edges, which the analyst can then plug into the standard network-econometrics pipeline.  The combination of a flexible first stage with debiased $\sqrt{n}$-inference makes the recovered network suitable as an input to downstream causal analyses without inheriting the slow convergence rates that would otherwise compromise the validity of second-stage estimates.

%% ════════════════════════════════════════════════════════════════════════════
\section{CascadeNet for Network Recovery}
\label{sec:cascade}

\subsection{The Problem}
\label{sec:problem}

We observe a panel of $C$ independent trajectories (cascades), where $C$ can be as small as one.  In each trajectory $c$, a population of $N$ agents (customers, regions, institutions) evolves according to
\begin{equation}
  \label{eq:dgp}
  Y_{c,t+1} = \m(Y_{c,t},\, X_{c,t};\, \theta) + \ve_{c,t+1},
  \qquad
  \E[\ve_{c,t+1} \mid Y_{c,t}, X_{c,t}, \theta] = 0,
\end{equation}
where $Y_{c,t} \in \R^N$ is the vector of agent states at time $t$, $X_{c,t} \in \R^{N \times d}$ collects agent covariates that may vary over time (e.g., demographics, promotional exposure, treatment assignments, or policy variables), $\theta$ denotes the parameters that govern the transition function, $\m$ is the unknown transition function, and $\ve_{c,t+1}$ are mean-zero idiosyncratic shocks.  The researcher observes the panel $\{Y_{c,t}, X_{c,t}\}$ for $c = 1, \ldots, C$ and $t = 1, \ldots, T$.

For instance, in a viral marketing setting, $Y_{c,t}$ records adoption status across customer segments and $X_{c,t}$ captures time-varying promotional exposure or static demographics; each cascade corresponds to a campaign or product launch, and $\m$ is the adoption response function that maps the current adoption pattern and marketing variables into next-period adoption probabilities.  In an epidemiological setting, $Y_{c,t}$ records disease incidence across regions and $X_{c,t}$ captures vaccination rates, mobility restrictions, or population density; each cascade corresponds to a wave of the epidemic, and $\m$ is the transmission function that maps current incidence and policy conditions into next-period incidence.

Several features of this formulation are worth highlighting.  First, the formulation is a general one in two senses: the state $Y_{c,t}$ can be binary (recording which agents have adopted by step $t$), continuous (representing case counts, sales, or other real-valued outcomes), or mixed; the transition function $\m$ is left unspecified, nesting the independent cascade, linear threshold, SIS, SIR, Hawkes, and other Markovian diffusion models as special cases.  Second, the parameter $\theta$ is fully flexible: in standard network models $\theta$ includes an $N \times N$ adjacency matrix, but we impose no such structure, and $\theta$ can encode any parametric or flexible specification of the transition function.  Third, the key identifying assumption is Markovian dynamics, namely that $Y_{c,t-1}$ affects $Y_{c,t+1}$ only through $Y_{c,t}$, which is standard in diffusion models; the framework extends to non-Markovian settings by augmenting the state to include lagged values (see Appendix~\ref{app:nonmarkov}).  Lastly, we treat the covariates $X_{c,t}$ as exogenous in order to focus on the network-recovery task, and leave extensions that accommodate endogenous covariates via instrumental or proxy variables to future work.

The goal of network recoveryis to recover the underlying network from this data.  We seek to understand the mechanism of influence in the network: given any state $Y_t$, how does a marginal change in agent $j$'s state affect agent $i$'s next-period outcome?  Existing methods answer this question through the parameters of a specific diffusion model, such as the pairwise transmission probabilities of a linear Markov model or the edge weights of an exponential kernel. We hightlight that, each of these is a special case of the Jacobian of the transition function, $J_0[i,j] = \partial m_{0,i} / \partial y_j$, which captures the marginal influence structure without imposing a specific parametric form.  We therefore target the Jacobian itself, which encodes the influence network in a flexible way that is robust to misspecification of the diffusion mechanism.  The asymptotic theory we develop targets the average Jacobian $\E[\partial m_{0,i}/\partial y_j]$, which summarizes the overall influence of $j$ on $i$ across the state space; the pointwise Jacobian at any given state is also informative and can be readily computed via automatic differentiation.  Analogously, the partial derivative $\partial m_{0,i}/\partial x_j$ captures the direct effect of covariates and can be estimated and debiased within the same framework.

\subsection{The CascadeNet Estimator}
\label{sec:setup}

We propose CascadeNet, a machine learning framework for network recovery.  The core idea is to estimate the transition function $m_0$ from the observed panel using a flexible machine learning estimator $\hat{m}$, then differentiate the estimated function and apply a debiasing correction to recover the network Jacobian.  Concretely, the procedure has three steps: (i) choose a model class $\mathcal{M}$ for the transition function, (ii) fit $\hat{m} \in \mathcal{M}$ to predict $Y_{c,t+1}$ from $(Y_{c,t}, X_{c,t})$ across all observed time steps and cascades, and (iii) compute the average Jacobian of $\hat{m}$ and apply the Riesz debiasing correction described in Section~\ref{sec:debias}.  The model class $\mathcal{M}$ enters only through the prediction step; the inference theory we develop in Section~\ref{sec:theory} applies to any $\mathcal{M}$ that satisfies the smoothness and convergence-rate assumptions stated there.

Choosing $\mathcal{M}$ involves two practical considerations.  First, the class should respect the natural constraints that the data impose on the range of $\m$: if $Y_{c,t}$ is binary, the transition function should map into $[0,1]$ and is most naturally specified through a sigmoid link; if $Y_{c,t}$ is a count or rate, the transition function should be non-negative and is typically specified through an identity, $\mathrm{ReLU}$, or softplus link with appropriate scaling.  Second, $\mathcal{M}$ should be flexible enough to approximate the true transition dynamics but not so complex as to overfit the available cascade data.  In our experiments, a linear-index model (see more details below) with a flexible link function is sufficient to dominate every classical baseline on all nine synthetic DGPs and on the COVID-19 application; richer classes such as multilayer perceptrons or graph neural networks become useful when the data exhibit higher-order interactions, asymmetric attention patterns, or strongly nonlinear feedback that the linear index cannot capture.

For concreteness, we now describe a linear-index specification of the model class $\mathcal{M}$ that we use as the default in our experiments, together with a generic absorbing wrapper that can be applied on top of any specification when adoption is irreversible.  Richer specifications such as graph neural networks can be plugged into the same framework; we describe one such implementation in Appendix~\ref{pf:gnn_grad}.

\paragraph{Default specification: Linear-index model.}
The simplest specification, which we refer to as the linear-index model, takes the form
\begin{equation}
  \label{eq:cascadenet}
  \hat{m}_i(Y_t, X_{t,i}) = \ell\!\bigl(J\, Y_t + k_i\, x_{t,i} + b_i\bigr),
\end{equation}
where $J \in \R^{N \times N}$ is a learnable network weight matrix, $k_i$ captures the direct effect of covariates on agent $i$'s transition, $b \in \R^N$ is a bias vector, and $\ell(\cdot)$ is the link function (sigmoid for binary data; identity or a small MLP for continuous data).  The regularizer $R(m)$ is specified as $\ell_1$ or $\ell_2$ weight regularization on $J$.  Despite its simplicity, this specification is already substantially more general than the parametric baselines used in the cascade-inference literature: NetInf assumes a specific exponential kernel and pairwise independence, DANI assumes a linear Markov transition without a link function, and the classical linear-threshold and independent-cascade likelihoods further restrict the form of $\ell(\cdot)$ and the structure of $J$.  The linear index in \eqref{eq:cascadenet} relaxes all three restrictions: it allows arbitrary pairwise interaction weights, accommodates nonlinearity through the link function, and treats the network weight matrix as a single learnable object rather than a collection of separately estimated transmission probabilities.  A useful by-product of this form is that the Jacobian of $\hat{m}$ admits the closed-form expression $\partial \hat{m}_i / \partial y_j = \ell'(z_i)\, J[i,j]$.

\paragraph{Generic augmentation: Absorbing wrapper.}
For diffusion models in which adoption is irreversible, such as epidemiological settings where the infected state is absorbing or marketing settings where adopting customers remain adopters, the transition function must satisfy the constraint that once an agent adopts, it stays adopted.  Any unconstrained estimator $\hat{m}_i^{\mathrm{raw}}$ in the model class $\mathcal{M}$, whether the linear-index model in \eqref{eq:cascadenet} or a richer specification, can be augmented with an absorbing wrapper that enforces this constraint exactly:
\begin{equation}
  \hat{m}_i^{\mathrm{abs}}(Y_t, X_{t,i})
    = Y_{t,i} + (1 - Y_{t,i})\, \hat{m}_i^{\mathrm{raw}}(Y_t, X_{t,i}).
\end{equation}
The wrapper leaves already-adopted agents in the absorbed state with probability one and lets $\hat{m}_i^{\mathrm{raw}}$ govern the transition only for agents that have not yet adopted.  CascadeNet applies the wrapper when the training panel exhibits irreversible transitions; in applications where irreversibility is substantively known, the analyst can also impose the wrapper directly.

Once the model class $\mathcal{M}$ has been specified, CascadeNet estimates $m_0$ by minimizing a regularized prediction loss over the observed panel:
\begin{equation}
  \label{eq:training}
  \hat{m} \in \arg\min_{m \in \mathcal{M}}
    \frac{1}{CT}\sum_{c=1}^C\sum_{t=1}^{T-1}
    \mathcal{L}\bigl(Y_{c,t+1},\, m(Y_{c,t}, X_{c,t})\bigr)
    + \lambda\, R(m),
\end{equation}
where $\mathcal{L}$ is binary cross-entropy (for binary data) or mean squared error (for continuous data), $\lambda > 0$ is the regularization parameter, and $R(m)$ penalizes model complexity.  In practice we minimize \eqref{eq:training} by stochastic gradient descent (Adam), with $\lambda$ selected by held-out cross-validation on a fold of trajectories disjoint from the one used for the debiased estimator.  When the debiasing step is applied (Section~\ref{sec:debias}), $\hat{m}$ is fit on each cross-fitting fold separately, and the score is evaluated on the held-out fold; this avoids the standard bias from using the same data for nuisance estimation and final inference.

\subsection{Network Recovery and Inference}
\label{sec:tasks}
\label{sec:debias}
\label{sec:theory}

\subsubsection{Defining the target}
\label{sec:target}

The final and most important step of CascadeNet is to recover an influence network from the estimated transition function $\hat{m}$.  Because $\m$ maps the current state $Y_t$ into next-period outcomes $Y_{t+1}$, the natural local measure of influence is its Jacobian with respect to the current state:
\begin{equation}
  \label{eq:Jpoint}
  \frac{\partial \m}{\partial Y}(y, x)
    = \left[\frac{\partial m_{0,i}}{\partial y_j}(y, x)\right]_{i,j=1}^N.
\end{equation}
The entry $\partial m_{0,i}(y, x)/\partial y_j$ measures the marginal effect of agent $j$'s current state on agent $i$'s expected next-period outcome at the state-covariate configuration $(y, x)$.  This yields a model-free notion of directed influence that does not depend on whether the underlying diffusion mechanism is independent cascade, linear threshold, SIS/SIR, Hawkes, or some richer nonlinear process.

In nonlinear environments, however, influence is generally state-dependent: the effect of $j$ on $i$ may be strong in some regions of the state space and weak or zero in others, because of saturation, threshold effects, or absorbing states.  For that reason, we take as our primary network estimand the average Jacobian,
\begin{equation}
  \label{eq:J}
  \Jz
    := \E\!\left[\frac{\partial \m}{\partial Y}(Y_t, X_t)\right]
    = \left[\E\!\left[\frac{\partial m_{0,i}}{\partial y_j}(Y_t, X_t)\right]\right]_{i,j=1}^N.
\end{equation}
Entry $J_0[i,j]$ summarizes the average marginal influence of agent $j$ on agent $i$, averaging over the distribution of states and covariates realized in the data.  We target this average Jacobian because it converts a potentially heterogeneous, state-specific influence structure into a single interpretable edge-weight matrix.  This is often the economically relevant object: it provides a global ranking of influence pathways, supports comparison to external benchmarks such as mobility flows or observed ties, and serves as a stable summary for downstream decision-making.

Targeting the average Jacobian also brings a major statistical advantage.  Unlike a pointwise derivative evaluated at a specific $(y, x)$, the average Jacobian is a smooth population functional of the regression function $\m$, and this makes it amenable to debiasing and $\sqrt{n}$-valid inference.  The pointwise Jacobian remains useful descriptively and can always be computed from $\hat{m}$ by automatic differentiation, but our formal asymptotic results focus on the average Jacobian $J_0$.

Similarly, one can define average derivative effects with respect to the covariates.  If $x_{t,\ell}$ denotes the $\ell$-th covariate entering the transition function, the corresponding average direct-effect matrix is
\begin{equation}
  \label{eq:K}
  K_0[i,\ell]
    := \E\!\left[\frac{\partial m_{0,i}}{\partial x_{t,\ell}}(Y_t, X_t)\right],
\end{equation}
which can be estimated and debiased by the same argument developed below.  We focus on $J_0$ in the remainder of this section because recovering cross-agent influence is the central network problem; the analogous results for $K_0$ are immediate.

\subsubsection{Bias in the naive plug-in}
\label{sec:bias_plugin}

A natural estimator of $J_0[i,j]$ is the plug-in derivative $\E[\partial \hat{m}_i / \partial y_j(Y_t, X_t)]$.  The problem is that $\hat{m}$ is estimated by regularized machine learning.  Regularization is necessary to control variance and avoid overfitting in high-dimensional and nonlinear settings, but it introduces bias: $\ell_1$ and $\ell_2$ penalties shrink the fitted model toward simpler functions, and this shrinkage attenuates the corresponding derivatives.  In the present application, this means the naive Jacobian estimator can systematically understate edge strengths.  To see this, fix a pair $(i, j)$ and define $\thetaij^0 := \E[\partial m_{0,i}/\partial y_j(Y_t, X_t)]$.  We can decompose the target as
\begin{equation}
  \label{eq:bias_decomp}
  \thetaij^0
    = \underbrace{\E\!\left[\frac{\partial \hat{m}_i}{\partial y_j}(Y_t, X_t)\right]}_{\text{plug-in (biased)}}
      + \underbrace{\E\!\left[\frac{\partial (m_{0,i} - \hat{m}_i)}{\partial y_j}(Y_t, X_t)\right]}_{\text{bias term}}.
\end{equation}
The first term is the naive plug-in estimator.  The second term is the bias in the derivative, induced by estimation error in $\hat{m}_i$.  This second term is the key difficulty: it depends on the unknown regression function $m_0$, and is therefore not directly observable.  Debiased network recovery reduces to finding a way to represent and estimate this bias term using observable quantities.

\subsubsection{Riesz representer and the orthogonal score}
\label{sec:riesz_score}

The central difficulty in \eqref{eq:bias_decomp} is that the bias term depends on the unknown regression function $m_0$, and is therefore not directly observable.  Our strategy is to rewrite this derivative bias in a form that can be estimated from data.  The device that makes this possible is the Riesz representer \citep{chernozhukov2022automatic, hirshberg2021augmented}, which converts the derivative functional into an $L^2$ inner product and allows the bias term to be expressed using observable prediction residuals.  This alone is not enough, however, because the Riesz representer must itself be estimated.  To control the resulting first-stage error, we combine the plug-in derivative and the Riesz correction in a Neyman-orthogonal score \citep{neyman1959, chernozhukov2018double}, so that small errors in estimating either nuisance affect the final estimator only at second order.  We now introduce these two ingredients in turn.

We begin with the Riesz representer, which provides the key representation of the derivative bias.  The construction in this subsection is not original to this paper: the Riesz representer for derivative functionals goes back to classical semiparametric theory and has been developed in its modern form by \citet{chernozhukov2022automatic} and \citet{hirshberg2021augmented}; we adopt their formulation directly and merely apply it to the network Jacobian.

\begin{definition}[Riesz Representer; \citealp{chernozhukov2022automatic, hirshberg2021augmented}]
  \label{def:riesz}
  For each pair $(i, j)$, the Riesz representer $\alphaij$ satisfies
  \begin{equation}
    \label{eq:riesz_def}
    \E[\alphaij(\Zt)\, h(\Zt)] = \E\!\left[\frac{\partial h(\Zt)}{\partial y_j}\right]
  \end{equation}
  for all sufficiently smooth test functions $h$ in a function class $\mathcal{H}$ for which $h \mapsto \E[\partial_j h(\Zt)]$ is a bounded linear functional, where $\Zt = (Y_t, X_t)$.
\end{definition}

This definition says that the derivative operator can be represented as an $L^2$ inner product with the function $\alphaij$.  Intuitively, instead of differentiating a function directly, we can multiply it by $\alphaij$ and take expectations to recover the same object.  Applying Definition~\ref{def:riesz} with $h(\Zt) = m_{0,i}(\Zt) - \hat{m}_i(\Zt)$, the bias term in \eqref{eq:bias_decomp} becomes
\begin{equation}
  \label{eq:bias_riesz}
  \E\!\left[\frac{\partial (m_{0,i} - \hat{m}_i)}{\partial y_j}(\Zt)\right]
    = \E\!\bigl[\alphaij(\Zt)\bigl(m_{0,i}(\Zt) - \hat{m}_i(\Zt)\bigr)\bigr].
\end{equation}
This is the key representation: it rewrites the derivative bias as an expectation involving the prediction error $m_{0,i}(\Zt) - \hat{m}_i(\Zt)$.  Although $m_0$ itself is unobservable, we do observe $Y_{t+1,i} = m_{0,i}(\Zt) + \ve_{t+1,i}$ with $\E[\ve_{t+1,i} \mid \Zt] = 0$, so the unobserved prediction error can be replaced, in expectation, by the observable residual $Y_{t+1,i} - \hat{m}_i(\Zt)$.  Together, these two steps deliver what we will refer to as an orthogonal bias correction.  Informally, the correction has two desirable properties.  First, it is a bias correction: starting from the biased plug-in derivative, we add the term $\E[\alphaij(\Zt)(Y_{t+1,i} - \hat{m}_i(\Zt))]$, which estimates the bias from \eqref{eq:bias_riesz} using only quantities the analyst can compute from the data.  Second, it is orthogonal in the Neyman sense: small errors in the two estimated nuisance functions, the regression $\hat{m}$ and the Riesz weight $\hat{\alpha}_{ij}$, do not contaminate the corrected estimator at first order, so the resulting estimator is robust to slow convergence in either nuisance.  This is exactly the property that allows valid edge-level inference at the parametric $\sqrt{n}$ rate even though $m_0$ is estimated by flexible machine learning; we make this formal in the next subsection.

We now define the score that underlies debiased estimation.  For each pair $(i, j)$, let
\begin{equation}
  \label{eq:score}
  \psiij(\Zt, Y_{t+1};\, \theta_{ij}, m, \alpha)
    = \frac{\partial m_i(\Zt)}{\partial y_j}
      + \alphaij(\Zt)\bigl(Y_{t+1,i} - m_i(\Zt)\bigr)
      - \theta_{ij},
\end{equation}
where $\theta_{ij}$ is the target parameter, $m_i$ is the $i$-th coordinate of the regression function, and $\alphaij$ is the corresponding Riesz representer.  At the truth $(\thetaij^0, \m, \alphaij)$, this score satisfies Neyman orthogonality: first-order perturbations in either the regression nuisance $m_i$ or the Riesz nuisance $\alphaij$ do not affect the moment condition (see Appendix~\ref{pf:neyman}).  Intuitively, the first term is the biased plug-in derivative, and the second term uses prediction residuals to correct that bias in a way that is robust to first-stage estimation error.  This orthogonality is what allows edge-level inference at the parametric $\sqrt{n}$ rate even though $\m$ is estimated by flexible machine learning.

For the linear-index model in Section~\ref{sec:setup}, the Riesz representer admits a closed form.  Let $r_i(\Zt) = [Y_t\trans,\, x_{t,i}\trans,\, 1]\trans \in \R^{N+d+1}$ and let $u_j$ denote the basis vector corresponding to the $j$-th coordinate of $Y_t$.  Then
\begin{equation}
  \alphaij(\Zt) = r_i(\Zt)\trans\, C_i\inv\, u_j,
  \qquad
  C_i = \E[r_i(\Zt)\, r_i(\Zt)\trans].
\end{equation}
For richer model classes, $\alphaij$ can be estimated numerically using the loss-based method of \citet{chernozhukov2022automatic}.

\subsubsection{Algorithm and inference}
\label{sec:algorithm_inference}

We now summarize the full debiased network-recovery procedure.

\begin{figure}[ht]
\centering
\begin{minipage}{0.95\textwidth}
\hrule\smallskip
\noindent\textbf{Algorithm~1.}\ \ CascadeNet Debiased Network Recovery
\smallskip\hrule\smallskip
\noindent\textit{Input:} panel $\{(Y_{c,t}, X_{c,t})\}_{c \le C,\, t \le T_c}$; class $\mathcal{M}$; folds $K$; regularization $\lambda$.

\noindent\textit{Output:} $\hat{J}^{\mathrm{deb}}$, standard errors, CIs, and edge $p$-values.
\smallskip\hrule\smallskip
\begin{enumerate}[leftmargin=2.2em,itemsep=1pt,topsep=2pt]
  \item Form transitions $(Z_{c,t}, Y_{c,t+1})$ with $Z_{c,t} = (Y_{c,t}, X_{c,t})$, and partition trajectories $c = 1, \ldots, C$ into $K$ folds $I_1, \ldots, I_K$.
  \item For each fold $k = 1, \ldots, K$, using only trajectories not in $I_k$:
  \begin{enumerate}[leftmargin=2em,itemsep=0pt,topsep=0pt,label=(\alph*)]
    \item fit $\hat{m}^{(-k)} \in \mathcal{M}$ by minimizing \eqref{eq:training};
    \item estimate the Riesz representers $\hat{\alpha}_{ij}^{(-k)}$ (closed form for the linear index, numerical loss otherwise);
    \item for each held-out $(c, t)$ with $c \in I_k$, evaluate
      $\hat{\psi}_{ij,c,t} = \partial_{y_j} \hat{m}_i^{(-k)}(Z_{c,t}) + \hat{\alpha}_{ij}^{(-k)}(Z_{c,t})\bigl(Y_{c,t+1,i} - \hat{m}_i^{(-k)}(Z_{c,t})\bigr)$.
  \end{enumerate}
  \item Aggregate to the trajectory level and average across folds:
    $\hat{\theta}_{ij}^{\mathrm{deb}} = n^{-1} \sum_{k} \sum_{c \in I_k} (T_c - 1)^{-1} \sum_{t=1}^{T_c - 1} \hat{\psi}_{ij,c,t}$, $n = C$.
  \item Set $\hat{J}^{\mathrm{deb}}[i,j] = \hat{\theta}_{ij}^{\mathrm{deb}}$; estimate the trajectory-level score variance and report standard errors, $100(1-\tau)\%$ CIs, and Wald tests for $H_0\!: J_0[i,j] = 0$.
\end{enumerate}
\smallskip\hrule
\end{minipage}
\end{figure}

Two implementation details deserve emphasis.  The first is the role of cross-fitting, which appears in step 2 of Algorithm~1.  If the same trajectories were used both to estimate the nuisance functions $(\hat{m}, \hat{\alpha})$ and to evaluate the score, the score would inherit a self-influence bias: any way in which $\hat{m}$ overfits a particular trajectory would also contaminate the residual computed on that trajectory, and the bias correction would be evaluated on the very residuals it was tuned to fit \citep{chernozhukov2018double}.  Cross-fitting eliminates this dependence by partitioning the cascades $c = 1, \ldots, C$ into $K$ folds, fitting $\hat{m}^{(-k)}$ and $\hat{\alpha}^{(-k)}$ on the $K - 1$ training folds, and evaluating the score only on the held-out fold $I_k$; averaging the resulting per-fold scores reproduces the trajectory-level CLT used in Theorem~\ref{thm:rate_JK}.  The partition is performed at the trajectory level rather than the observation level, which preserves within-trajectory dependence and matches the asymptotic regime where the number of trajectories $n = C$ grows.  The second detail is selection of the regularization parameter $\lambda$ in \eqref{eq:training}.  In our experiments $\lambda$ is chosen by held-out cross-validation on a fold of trajectories disjoint from the cross-fitting folds used to evaluate the score, so that the data used to tune $\lambda$ play no role in the inference step.  In practice this is implemented as an outer hyperparameter-tuning split nested inside each cross-fitting fold; the same Riesz-loss validation is used to tune $\hat{\alpha}_{ij}$ when it is estimated numerically.  Both choices are standard in the debiased-ML literature \citep{chernozhukov2018double, chernozhukov2022automatic} and are necessary for the asymptotic guarantees in Theorem~\ref{thm:rate_JK} to apply.

The orthogonal score delivers valid edge-level inference under standard regularity conditions.  For the asymptotic theory, we treat each cascade trajectory as an independent sampling unit and let the number of trajectories $n$ grow; cross-fitting is performed at the trajectory level.  For each trajectory $c$, define the trajectory-level score
\[
  \phi_{ij}(W_c)
    = \frac{1}{T_c - 1} \sum_{t=1}^{T_c - 1}
      \psiij\bigl(Z_{c,t}, Y_{c,t+1};\, \thetaij^0, \m, \alphaij\bigr),
\]
where $W_c = \{(Y_{c,t}, X_{c,t}, Y_{c,t+1})\}_{t=1}^{T_c - 1}$.

\begin{theorem}[Asymptotic Normality of the Debiased Jacobian]
  \label{thm:rate_JK}
  Under the regularity conditions listed in Assumptions~\ref{ass:iid}--\ref{ass:variance} in Appendix~\ref{app:assumptions}, for each fixed pair $(i, j)$,
  \begin{equation}
    \label{eq:clt_theta}
    \sqrt{n}\,\bigl(\hat{\theta}_{ij}^{\mathrm{deb}} - \thetaij^0\bigr)
      \xrightarrow{d} \mathcal{N}(0,\, V_{ij}),
    \qquad
    V_{ij} = \E[\phi_{ij}(W_c)^2].
  \end{equation}
  Moreover, the following cross-fitted sample analogue $\hat{V}_{ij}$ is a consistent estimator of $V_{ij}$:
  \begin{equation}
    \label{eq:vhat}
    \hat{V}_{ij}
      = \frac{1}{n} \sum_{k=1}^{K} \sum_{c \in I_k} \hat{\phi}_{ij}^{(k)}(W_c)^2,
    \qquad
    \hat{\phi}_{ij}^{(k)}(W_c)
      = \frac{1}{T_c - 1} \sum_{t=1}^{T_c - 1}
        \psiij\bigl(Z_{c,t}, Y_{c,t+1};\, \hattheta, \hat{m}^{(-k)}, \hat{\alpha}_{ij}^{(-k)}\bigr),
  \end{equation}
  where $I_k$ are the cross-fitting folds and $(\hat{m}^{(-k)}, \hat{\alpha}_{ij}^{(-k)})$ are the nuisance estimates fitted without fold $k$.
\end{theorem}

Theorem~\ref{thm:rate_JK} is a direct application of the debiased machine learning template of \citet{chernozhukov2018double} and the Riesz-representer framework of \citet{chernozhukov2022automatic} and \citet{hirshberg2021augmented}; our contribution is not the underlying inference machinery, which is by now well established in the econometrics and statistics literatures, but the specific identification of the network Jacobian $J_0$ as a smooth functional of the transition regression and the verification that the resulting orthogonal score is well behaved under repeated sampling of cascade trajectories.  The theorem implies that each entry of the debiased Jacobian admits asymptotically valid standard errors and confidence intervals, even though the underlying transition function $m_0$ is estimated by flexible machine learning.

Theorem~\ref{thm:rate_JK} yields immediate edge-level inference.  Let $\hat{V}_{ij}$ be the variance estimator in \eqref{eq:vhat}.  Then the standard error of $\hat{\theta}_{ij}^{\mathrm{deb}}$ is $\widehat{\mathrm{se}}(\hat{\theta}_{ij}^{\mathrm{deb}}) = \sqrt{\hat{V}_{ij}/n}$, an asymptotically valid $100(1 - \tau)\%$ confidence interval for the edge $(j \to i)$ is $\hat{\theta}_{ij}^{\mathrm{deb}} \pm z_{1 - \tau/2}\, \widehat{\mathrm{se}}(\hat{\theta}_{ij}^{\mathrm{deb}})$, and a test of $H_0\!: J_0[i,j] = 0$ is based on the Wald statistic $\hat{\theta}_{ij}^{\mathrm{deb}} / \widehat{\mathrm{se}}(\hat{\theta}_{ij}^{\mathrm{deb}})$.  These inferential objects are useful in several ways.  First, they allow the researcher to distinguish statistically meaningful edges from noise, rather than relying only on point estimates or heuristic rankings.  Second, they allow uncertainty-aware edge ranking, which is especially important when the recovered network will be used for targeting, intervention design, or downstream policy analysis.  Third, they make it possible to communicate uncertainty transparently in applications where decisions based on inferred influence links are consequential.

Theorem~\ref{thm:rate_JK} is stated under correct specification of the transition function, meaning that $m_0$ belongs to the model class $\mathcal{M}$ used to estimate $\hat{m}$.  When the model is misspecified, the same debiasing argument delivers asymptotic normality around the average Jacobian of the pseudo-true regression function, that is, the best approximation to $m_0$ within the chosen class, rather than around the true $J_0$.  The method therefore remains well behaved under misspecification, but inherits an approximation bias whose magnitude depends on the expressive power of the chosen model class.  Appendix~\ref{app:misspec} provides the formal argument.

%% ════════════════════════════════════════════════════════════════════════════
\section{Validation}
\label{sec:validation}

We validate CascadeNet on both synthetic and real data.  The synthetic exercise allows us to assess recovery performance under controlled conditions where the data-generating process is known and the true average Jacobian can be computed.  The empirical application assesses whether the recovered network aligns with an external benchmark in a real-world epidemiological setting.

\subsection{Synthetic Validation}
\label{sec:simulation}

\subsubsection{Simulation}

We evaluate CascadeNet on synthetic cascades generated from a known network under a range of diffusion mechanisms.  The goal is to test whether the method can recover the DGP-implied influence structure when the diffusion model is correctly specified, mildly misspecified, or strongly misspecified relative to classical baselines.

We generate a directed Erd\H{o}s--R\'{e}nyi network with $N = 64$ nodes and edge probability $p = 0.05$, yielding on average approximately 195 directed edges.  Edge weights are drawn independently from $[0.2, 0.8]$.

We pair this network with nine data-generating processes that span three broad families.  The first family, pairwise transmission, captures settings in which influence flows one edge at a time: each active node attempts to activate each of its neighbors, either independently (Independent Cascade) or once a weighted sum crosses a node-specific threshold (Linear Threshold).  The second family, continuous-time SI, models infection as the result of waiting-time hazards along edges, with Exponential-SI and Power-law-SI corresponding to different shapes of that hazard.  The third family, aggregate-effect and nonlinear dynamics, departs from clean pairwise propagation: SIS and SIR allow recovery and (for SIR) absorbing immunity, Complex Contagion requires multiple simultaneously active neighbors before a node activates, Hawkes generates self-exciting bursts whose intensity depends on past events, and a flexible Nonlinear DGP combines polynomial, oscillatory, and threshold-like features.  Together these nine DGPs span absorbing versus non-absorbing dynamics, pairwise versus aggregate transmission, and both well-specified and strongly misspecified environments for the competing methods.  We rely on the classical formulations developed in \citet{kempe2003maximizing}, \citet{gomezrodriguez2011uncovering}, \citet{kermack1927contribution}, \citet{centola2007complex}, and \citet{hawkes1971spectra}; the explicit one-step update equation and parameter choices for each DGP are given in Appendix~\ref{app:dgp}.

For each DGP, we generate $C_{\mathrm{train}} = 1{,}000$ training cascades of length $T = 10$--$30$, each initialized from two randomly chosen seed nodes.  We evaluate each method on an independent set of $C_{\mathrm{test}} = 1{,}000$ cascades.  To compute the true population target $J_0$, we use a separate held-out sample of $C_{\mathrm{eval}} = 5{,}000$ cascades.

We compare six approaches.  Pairwise is a simple temporal-precedence heuristic that scores edge $j \to i$ by the fraction of cascades in which node $j$ becomes active before node $i$ \citep[in the spirit of][]{gomezrodriguez2010inferring}.  NetInf \citep{gomezrodriguez2010inferring} performs likelihood-based network inference under an exponential independent-cascade kernel.  NETRATE \citep{gomezrodriguez2011uncovering} estimates pairwise transmission rates from first-infection timestamps by convex optimization.  DANI \citep{ramezani2024dani} estimates pairwise transmission probabilities from the observed panel using a Markov-transition formulation.  LTMLE is longitudinal targeted minimum-loss estimation \citep{vanderlaan2012ltmle} adapted to panel cascade data.  CascadeNet is our method with sigmoid link, $\lambda = 10^{-4}$, 300 training epochs, and 2-fold Riesz debiasing.

Because the estimand in Section~\ref{sec:tasks} is the average Jacobian rather than the binary adjacency matrix, our primary simulation benchmark compares estimated edge scores to the true DGP-implied mean Jacobian
\[
  J_0[i,j] = \E\!\left[\frac{\partial m_{0,i}}{\partial y_j}(Y_t, X_t)\right].
\]
Different diffusion models can induce very different marginal influence magnitudes even on the same underlying graph: an edge that is important under IC need not have the same average effect under Hawkes or SIS dynamics.  For that reason, adjacency recovery alone is not the right evaluation target for our estimator.

We compute the true $J_0$ numerically from each DGP's conditional expectation operator, using finite differences evaluated over the independent held-out sample of $5{,}000$ cascades.  Our primary metric is the Pearson correlation between each method's estimated edge-score vector and $|J_0|$, using off-diagonal entries only.  We focus on $|J_0|$ because the network-recovery task is primarily about ranking influence strength, and in these simulations the economically relevant comparison is the magnitude of marginal influence rather than its sign.  As a secondary metric, we report precision-recall performance for detecting the support of the underlying adjacency matrix $W$.  This secondary benchmark is useful for comparing topology recovery, but our primary criterion remains alignment with the true average Jacobian, since that is the estimand delivered by CascadeNet.

Table~\ref{tab:task1} reports the Pearson correlation between estimated edge scores and the true mean Jacobian $|J_0|$ across all nine diffusion models.  CascadeNet achieves the highest correlation in every setting.

\begin{table}[ht]
\centering
\caption{Network Recovery: Pearson $r$ between estimated scores and
  true Jacobian $|J_0|$ ($C = 1{,}000$ training cascades, ER network
  with $N = 64$, $|E| = 195$).  Bold indicates best method per row.
  $\Delta$ is the improvement of CascadeNet over the best baseline;
  stars denote significance via Fisher $z$-test:
  {*}{*}{*}\,$p < 0.001$.}
\label{tab:task1}
\small
\begin{tabular}{l ccccc c c}
\toprule
 & \multicolumn{5}{c}{Baselines} & CascadeNet & \\
\cmidrule(lr){2-6} \cmidrule(lr){7-7}
DGP & Pairwise & NetInf & NETRATE & DANI & LTMLE & Ours & $\Delta$ \\
\midrule
IC        & 0.32 & 0.65 & 0.64 & 0.47 & 0.66 & \textbf{0.84} & $+$0.18$^{***}$ \\
LT        & 0.45 & 0.70 & 0.63 & 0.58 & 0.68 & \textbf{0.78} & $+$0.08$^{***}$ \\
Exp-SI    & 0.37 & 0.50 & 0.31 & 0.55 & 0.73 & \textbf{0.80} & $+$0.07$^{***}$ \\
PL-SI     & 0.43 & 0.32 & 0.36 & 0.60 & 0.77 & \textbf{0.81} & $+$0.04$^{***}$ \\
SIS       & 0.10 & 0.58 & 0.44 & 0.26 & 0.29 & \textbf{0.86} & $+$0.28$^{***}$ \\
SIR       & 0.14 & 0.51 & 0.50 & 0.49 & 0.64 & \textbf{0.84} & $+$0.21$^{***}$ \\
Complex   & 0.41 & 0.57 & 0.33 & 0.59 & 0.70 & \textbf{0.80} & $+$0.10$^{***}$ \\
Hawkes    & $-$0.04 & 0.32 & 0.30 & 0.19 & 0.20 & \textbf{0.82} & $+$0.50$^{***}$ \\
Nonlinear & 0.23 & 0.27 & 0.25 & 0.25 & 0.46 & \textbf{0.54} & $+$0.08$^{***}$ \\
\bottomrule
\end{tabular}
\end{table}

Two patterns stand out.  First, the classical baselines are model-sensitive: each tends to perform well on the family most closely aligned with its maintained assumptions but degrades sharply outside that family.  NetInf performs strongly on LT but much worse on Hawkes; LTMLE performs well on the continuous-time SI models but deteriorates on Hawkes and the flexible nonlinear DGP; Pairwise performs especially poorly on epidemic models with recovery/removal, such as SIS and SIR.  Second, CascadeNet is uniformly strong: its correlation never falls below $0.54$, and it delivers the best edge ranking in every environment considered.

Figure~\ref{fig:pearson_bars} visualizes these comparisons.  We also report precision-recall curves for adjacency recovery in Figure~\ref{fig:pr_curves} (Appendix~\ref{app:pr_curves}).  CascadeNet maintains substantially higher precision at moderate-to-high recall on the aggregate-effect and nonlinear models, while remaining competitive with the strongest classical baselines on pairwise-transmission DGPs where those methods are relatively well specified.

\begin{figure}[ht]
  \centering
  \includegraphics[width=\textwidth]{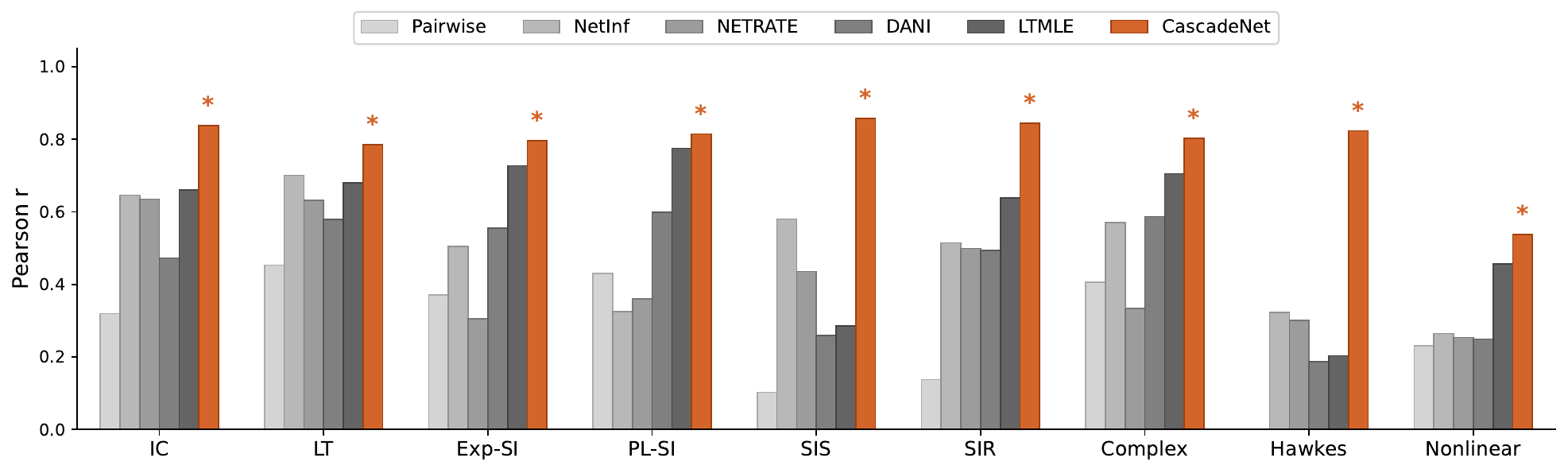}
  \caption{Pearson correlation between estimated scores and true
    Jacobian $|J_0|$ across nine diffusion models ($C = 1{,}000$,
    ER network).  CascadeNet achieves the highest correlation in all
    nine settings.  Classical baselines are inconsistent: each excels
    on its ``native'' DGP family but fails on others.}
  \label{fig:pearson_bars}
\end{figure}

Taken together, the simulation results support the main claim of the paper: when the true diffusion mechanism is unknown, a flexible Jacobian-based estimator with orthogonal debiasing provides a more robust recovery strategy than methods built around a single parametric propagation model.

\subsection{Empirical Application: COVID-19 Transmission in Spain}
\label{sec:covid}

We now apply CascadeNet to a real-world epidemiological setting: recovering inter-province transmission structure from daily COVID-19 case counts in Spain.

During the COVID-19 pandemic, a central policy problem was to understand which regions were transmitting infection to which other regions.  This information was relevant for targeted interventions such as regional travel restrictions, testing deployment, and surveillance prioritization.  While the transmission network itself is not observed, inter-province mobility flows provide a useful external benchmark for the plausibility of the recovered edge rankings.  Mobility is not a literal ground truth for transmission, because behavioral adaptation, public-health policy, and local conditions all affect realized contagion, but stronger alignment with mobility is suggestive evidence that the inferred network is capturing meaningful transmission opportunities.

We use the dataset of \citet{murphy2021deep}, which combines daily COVID-19 case counts from Spain's Centro Nacional de Epidemiolog\'{i}a with inter-province mobility flows from the Ministerio de Fomento.  The case panel records daily new infections for $N = 52$ provinces over $T = 450$ days, from January 2020 to March 2021.  Following \citet{murphy2021deep}, we normalize each province's incidence series by its mean daily incidence to improve comparability across provinces of different sizes.

To construct multiple panels, we use sliding windows of 56 days with stride 28, yielding $C = 15$ partially overlapping trajectory windows.  Because these windows overlap, the formal independent-trajectory asymptotic theory of Section~\ref{sec:tasks} should be viewed here as suggestive rather than exact.  We therefore interpret this application primarily as an external validation exercise for the recovered edge rankings, rather than as a literal implementation of the asymptotic sampling regime used in the theorem.

As an external benchmark, we construct a thresholded mobility network from inter-province traveler flows.  Following \citet{murphy2021deep}, we log-transform the flows, symmetrize them, and threshold the resulting matrix to an average degree of approximately 10, yielding $|E| = 519$ undirected benchmark edges.  In the evaluation below, we compare off-diagonal estimated edge scores to this benchmark.

Table~\ref{tab:task2} reports empirical performance.  We evaluate whether estimated edge rankings align with the benchmark mobility network using the break-even point (BEP) of the precision-recall curve, AUC-PR, and the Pearson correlation between estimated edge scores and log mobility flow.

\begin{table}[ht]
\centering
\caption{Network Recovery on Spain COVID-19 Data ($N = 52$ Provinces,
  $|E| = 519$ Mobility Edges).  Pearson $r$ measures correlation
  between estimated edge scores and log inter-province mobility flow.
  Bold indicates best method.}
\label{tab:task2}
\small
\begin{tabular}{l cc cc}
\toprule
 & & & \multicolumn{2}{c}{Pearson $r$} \\
\cmidrule(lr){4-5}
Method & BEP & AUC-PR & $r$ & $p$-value \\
\midrule
Pairwise    & 0.194 & 0.202 & $-0.012$ & $0.53$ \\
Granger     & 0.236 & 0.209 & $+0.025$ & $0.20$ \\
NetInf      & 0.174 & 0.186 & $+0.004$ & $0.83$ \\
NETRATE     & 0.190 & 0.188 & $+0.024$ & $0.22$ \\
DANI        & 0.176 & 0.196 & $-0.011$ & $0.58$ \\
LTMLE       & 0.176 & 0.195 & $-0.012$ & $0.52$ \\
CascadeNet  & \textbf{0.298} & \textbf{0.263} & $\mathbf{+0.139}$ & $< 10^{-12}$ \\
\bottomrule
\end{tabular}
\end{table}

CascadeNet achieves the highest BEP and AUC-PR among all methods.  The benchmark-network density is $0.196$, which is the random baseline for precision-recall evaluation; most baselines lie only slightly above that level, whereas CascadeNet improves on it substantially.  On the more demanding continuous benchmark, CascadeNet is also the only method whose estimated edge scores correlate significantly with log mobility flows.  The baseline correlations are all close to zero.

This result is notable because CascadeNet is trained only on case-count dynamics and never observes mobility data during estimation.  The positive alignment with mobility therefore emerges from the recovered transmission structure itself, rather than from any direct use of the benchmark.

The empirical benchmark is intentionally demanding.  Mobility is only an imperfect proxy for transmission: lockdowns, masking, behavioral responses, underreporting, and within-province spread all weaken the mapping from traveler flows to realized infections.  The relevant question is therefore not whether a method perfectly reconstructs the benchmark network, but whether it extracts any meaningful cross-region transmission signal from noisy epidemic data.  By that standard, CascadeNet appears to succeed where the classical baselines do not.  Its moderate but clearly positive alignment with mobility suggests that the flexible, Jacobian-based approach is recovering transmission pathways that are plausibly related to real inter-provincial contact intensity, even in an environment where the true contagion mechanism is far more complex than the kernels assumed by standard cascade-inference methods.

%% ════════════════════════════════════════════════════════════════════════════
\section{Managerial Implications and Conclusion}
\label{sec:managerial}
\label{sec:conclusion}

This paper develops CascadeNet, a flexible framework for network recovery from cascade data.  The method is built on a simple but general insight: under a Markovian diffusion process, the Jacobian of the one-step transition function summarizes how one agent's current state affects another agent's next-period outcome.  By estimating this transition function flexibly and then applying Neyman-orthogonal debiasing via the Riesz representer, CascadeNet recovers an average influence network while correcting the regularization bias that would otherwise distort the naive plug-in Jacobian.  The result is a network estimator that combines flexibility with formal statistical inference on individual edges.

The managerial implications are substantial.  In many practical settings, the analyst does not know how influence propagates: a marketing manager may not know whether consumers respond to pairwise persuasion, peer reinforcement, or threshold behavior, and a public-health authority faces transmission dynamics shaped by mobility, policies, and behavioral adaptation that no single parametric model captures well.  Our results show that this uncertainty matters: methods tailored to one propagation model can deteriorate sharply when the assumed kernel is wrong, whereas CascadeNet learns the transition function directly from the data without requiring the analyst to commit ex ante to a specific diffusion mechanism.  This robustness is particularly valuable because many downstream decisions, such as seed selection in marketing \citep{kempe2003maximizing}, travel restrictions and testing allocation in epidemiology, and contagion monitoring in financial settings, depend on identifying which links in the network are most consequential.  In all of these applications, ranking edges by estimated importance is useful, but ranking them with confidence intervals is substantially more valuable.

The validation results support these conclusions.  Across nine synthetic diffusion models spanning pairwise, epidemic, aggregate-effect, and nonlinear environments, CascadeNet achieves the highest recovery accuracy in every setting, and in the Spain COVID-19 application it is the only method whose recovered edge scores align significantly with inter-province mobility flows, even though mobility data are never used in estimation.  The debiasing step is central to this performance: when cascade data are sparse and regularization is strong, the Riesz-based correction is often the difference between recovering a useful network signal and recovering almost none at all (Appendix~\ref{app:debiasing_sim}).

Several extensions remain for future work, including adaptive topology learning when the adjacency structure is fully unknown, uniform inference for many edges simultaneously, continuous-time diffusion with flexible kernels, and inference under single-cascade or dependent-trajectory sampling.  Finally, predicting long-run or steady-state outcomes from the estimated transition function is a natural next step; we pursue this direction in companion work on equilibrium learning, where the core challenge is valid inference on the fixed point of a flexibly estimated map.

More broadly, the paper's message is that network recovery need not force a tradeoff between flexibility and inferential rigor.  By combining modern machine learning with orthogonal debiasing, CascadeNet makes recovered networks not just descriptive objects, but reliable inputs into managerial and policy analysis.

%% ════════════════════════════════════════════════════════════════════════════
\bibliographystyle{apalike}

%% ════════════════════════════════════════════════════════════════════════════
\appendix
\renewcommand{\thesection}{\Alph{section}}
\newcommand{\appsec}[2]{%
  \refstepcounter{section}%
  \label{#2}%
  \section*{Appendix~\thesection: #1}%
  \addcontentsline{toc}{section}{Appendix~\thesection: #1}%
}

\appsec{Assumptions}{app:assumptions}

The assumptions below adapt the standard regularity conditions of the debiased machine learning literature, in particular \citet{chernozhukov2018double} and the Riesz-representer framework of \citet{chernozhukov2022automatic} and \citet{hirshberg2021augmented}, to the panel structure of cascade data.  None of these conditions are new to this paper; they are imposed verbatim or with minor modifications relative to those references.  For the asymptotic theory, we treat each cascade trajectory as an independent sampling unit.  Let
\[
  W_c = \{(Y_{c,t}, X_{c,t}, Y_{c,t+1})\}_{t=1}^{T_c - 1},
  \qquad c = 1, \ldots, n,
\]
denote the $c$-th observed trajectory, where $n = C$ is the number of cascades, and define $Z_{c,t} := (Y_{c,t}, X_{c,t})$.  For each fixed pair $(i, j)$, recall that the target parameter is
\[
  \thetaij^0
    := \E\!\left[\frac{\partial m_{0,i}}{\partial y_j}(Y_t, X_t)\right].
\]
We impose the following assumptions.

\begin{assumption}[Independent trajectories]
  \label{ass:iid}
  The trajectories $W_1, \ldots, W_n$ are independent and identically distributed.  Their lengths satisfy $2 \le T_c \le \bar{T} < \infty$ for all $c$, for some fixed constant $\bar{T}$.
\end{assumption}

\begin{assumption}[Smoothness and bounded moments]
  \label{ass:smooth}
  For each coordinate $i$, the regression function $m_{0,i}(z)$ is continuously differentiable in $Y$.  For each pair $(i, j)$, $\partial m_{0,i}/\partial y_j(Z_{c,t}) \in L^2(P_Z)$ and $\E\bigl[(\partial m_{0,i}/\partial y_j(Z_{c,t}))^2\bigr] < \infty$.  Moreover, the regression residual $\ve_{c,t+1,i} := Y_{c,t+1,i} - m_{0,i}(Z_{c,t})$ satisfies $\E[\ve_{c,t+1,i} \mid Z_{c,t}] = 0$ and $\E[\ve_{c,t+1,i}^2] < \infty$.
\end{assumption}

\begin{assumption}[Riesz existence and boundedness]
  \label{ass:riesz}
  For each fixed pair $(i, j)$, the linear functional $h \mapsto \E[\partial h(\Zt)/\partial y_j]$ is bounded on the function class $\mathcal{H}$ in Definition~\ref{def:riesz}.  Hence there exists a unique Riesz representer $\alphaij \in L^2(P_Z)$ such that $\E[\alphaij(\Zt)\, h(\Zt)] = \E[\partial h(\Zt)/\partial y_j]$ for all $h \in \mathcal{H}$, and $\E[\alphaij(\Zt)^2] < \infty$.
\end{assumption}

\begin{assumption}[Cross-fitted nuisance rates]
  \label{ass:rates}
  Let $\hat{m}^{(-k)}$ and $\hat{\alpha}_{ij}^{(-k)}$ denote the nuisance estimators fit on the complement of fold $I_k$.  Then for each fold $k$,
  \[
    \|\hat{m}_i^{(-k)} - m_{0,i}\|_{L^2(P_Z)} = o_p(n^{-1/4}),
    \qquad
    \|\hat{\alpha}_{ij}^{(-k)} - \alphaij\|_{L^2(P_Z)} = o_p(n^{-1/4}),
  \]
  and
  \[
    \left\|\frac{\partial \hat{m}_i^{(-k)}}{\partial y_j} - \frac{\partial m_{0,i}}{\partial y_j}\right\|_{L^2(P_Z)} = o_p(1).
  \]
\end{assumption}

\begin{assumption}[Finite trajectory-level score variance]
  \label{ass:variance}
  For each fixed pair $(i, j)$, define the oracle trajectory-level score
  \[
    \phi_{ij}(W_c)
      = \frac{1}{T_c - 1} \sum_{t=1}^{T_c - 1}
        \psiij\!\bigl(Z_{c,t}, Y_{c,t+1};\, \thetaij^0, \m, \alphaij\bigr),
  \]
  where $\psiij$ is the score in \eqref{eq:score}.  Then $\E[\phi_{ij}(W_c)] = 0$ and $0 < V_{ij} := \E[\phi_{ij}(W_c)^2] < \infty$.
\end{assumption}

Assumptions~\ref{ass:iid}--\ref{ass:variance} are standard in debiased machine learning \citep{chernozhukov2018double, chernozhukov2022automatic, hirshberg2021augmented}, adapted here to repeated sampling of cascade trajectories.  Assumption~\ref{ass:rates} in particular is the now-canonical $o_p(n^{-1/4})$ product-rate condition introduced by \citet{chernozhukov2018double}, which ensures that the product of nuisance errors is $o_p(n^{-1/2})$, which is sufficient for valid debiasing.  We do not view these assumptions as a contribution of this paper; the contribution is the application of this established framework to a setting (the network Jacobian recovered from cascade data) where it has not previously been deployed.

\appsec{Diffusion-Model Specifications for the Synthetic Validation}{app:dgp}

This appendix describes the nine data-generating processes used in Section~\ref{sec:simulation}.  Throughout, $W \in [0,1]^{N \times N}$ denotes the directed weighted adjacency matrix described in Section~\ref{sec:simulation} (with $W_{ji}$ the weight on the directed edge $j \to i$), $Y_{c,t} \in \{0,1\}^N$ denotes the activation state of the $N$ agents at time $t$ in cascade $c$, and $\sigma(\cdot)$ denotes the logistic sigmoid.  All cascades start from two randomly chosen seed nodes.  We summarize the one-step transition that defines each DGP.

\paragraph{Independent Cascade (IC) \citep{kempe2003maximizing}.}
Adoption is irreversible.  At each step, every newly active node $j$ at time $t$ makes a single attempt to activate each currently inactive neighbor $i$, succeeding independently with probability $W_{ji}$:
\[
  \Pr(Y_{c,t+1,i} = 1 \mid Y_{c,t})
    = 1 - \prod_{j \in N_{\mathrm{new}}(t)}(1 - W_{ji}),
  \qquad
  Y_{c,t+1,i} \ge Y_{c,t,i}.
\]
We set diagonal entries of $W$ to zero so that nodes do not self-activate.

\paragraph{Linear Threshold (LT) \citep{kempe2003maximizing}.}
Each node $i$ has a threshold $\theta_i$ drawn uniformly from $(0, 1]$.  Adoption is irreversible.  At time $t+1$,
\[
  Y_{c,t+1,i}
    = \mathbf{1}\!\left\{\sum_{j} W_{ji}\, Y_{c,t,j} \ge \theta_i\right\}
      \;\vee\; Y_{c,t,i}.
\]
We rescale rows of $W$ so that $\sum_j W_{ji} \le 1$ for each $i$, ensuring well-defined thresholds.

\paragraph{Exponential-SI \citep{gomezrodriguez2011uncovering}.}
Continuous-time SI with exponential transmission times.  The waiting time for an active $j$ to infect an inactive $i$ is exponential with rate $\lambda_{ji} = W_{ji}$.  We discretize time with unit step $\Delta t = 1$ and obtain the one-step activation probability
\[
  \Pr(Y_{c,t+1,i} = 1 \mid Y_{c,t})
    = 1 - \prod_{j: Y_{c,t,j} = 1}\!\exp\bigl(-W_{ji}\, \Delta t\bigr),
  \qquad
  Y_{c,t+1,i} \ge Y_{c,t,i}.
\]

\paragraph{Power-law-SI \citep{gomezrodriguez2011uncovering}.}
Same as Exponential-SI but with a power-law transmission hazard $h_{ji}(\tau) = (\alpha - 1)\, \tau^{-\alpha}\, W_{ji}$ for elapsed time $\tau$ since $j$'s infection.  We use $\alpha = 2.0$ and discretize as in Exponential-SI.

\paragraph{SIS \citep{kermack1927contribution}.}
Susceptible--Infected--Susceptible dynamics with infection from active neighbors and recovery rate $\gamma$:
\[
  \Pr(Y_{c,t+1,i} = 1 \mid Y_{c,t})
    = (1 - Y_{c,t,i})\, \Bigl[1 - \prod_{j}(1 - W_{ji}\, Y_{c,t,j})\Bigr]
      + Y_{c,t,i}\,(1 - \gamma).
\]
We set $\gamma = 0.3$.  Adoption is reversible (non-absorbing).

\paragraph{SIR \citep{kermack1927contribution}.}
Susceptible--Infected--Recovered dynamics.  Infected nodes recover (and become permanently immune) with rate $\gamma$:
\[
  Y_{c,t+1,i}
    = \begin{cases}
        \mathbf{1}\{\text{at least one active neighbor infects } i\} & \text{if } Y_{c,t,i} = 0 \text{ and } i \text{ is susceptible},\\
        Y_{c,t,i} & \text{otherwise (including recovered)},
      \end{cases}
\]
where the infection probability is $1 - \prod_j (1 - W_{ji}\, Y_{c,t,j})$ and recovery occurs independently each step with probability $\gamma = 0.15$.

\paragraph{Complex Contagion \citep{centola2007complex}.}
A node activates only when at least $k$ of its weighted neighbors are simultaneously active.  Adoption is irreversible.  We set $k = 4$:
\[
  Y_{c,t+1,i}
    = \mathbf{1}\!\left\{\sum_{j} \mathbf{1}\{W_{ji} > 0\}\, Y_{c,t,j} \ge k\right\}
      \;\vee\; Y_{c,t,i}.
\]

\paragraph{Hawkes (self-exciting) \citep{hawkes1971spectra}.}
Each active node generates an excitation that decays exponentially in time and triggers offspring activations on its out-neighbors.  We discretize so that the activation probability of $i$ at time $t+1$ is
\[
  \Pr(Y_{c,t+1,i} = 1 \mid Y_{c,t})
    = \sigma\!\Bigl(\beta_0 + \beta_1 \sum_{s \le t} \sum_j W_{ji}\, e^{-\kappa(t - s)}\, \mathbf{1}\{j \text{ activated at } s\}\Bigr),
\]
with decay $\kappa = 0.5$ and intensities $\beta_0 = -3$, $\beta_1 = 1$.  Adoption is reversible.

\paragraph{Nonlinear DGP.}
A flexible specification designed to violate the parametric assumptions of every classical baseline.  The transition is sigmoid-of-nonlinear-features:
\[
  \Pr(Y_{c,t+1,i} = 1 \mid Y_{c,t})
    = \sigma\!\bigl(b_i + (W Y_{c,t})_i + \tfrac{1}{2} (W Y_{c,t})_i^3
                    + \sin\!\bigl(2 (W Y_{c,t})_i\bigr)
                    + \mathrm{ReLU}\!\bigl((W Y_{c,t})_i - 0.5\bigr)\bigr),
\]
with bias $b_i = -1$.  This DGP combines polynomial, oscillatory, and threshold-like features to test recovery under strongly misspecified competing methods.

\bigskip
For each DGP we generate $C_{\mathrm{train}} = 1{,}000$ training cascades of length $T = 10$--$30$, $C_{\mathrm{test}} = 1{,}000$ test cascades, and $C_{\mathrm{eval}} = 5{,}000$ held-out cascades used to compute the true mean Jacobian $J_0$ via finite differences on the conditional expectation operator.

\appsec{Additional Figures}{app:pr_curves}

\begin{figure}[ht]
  \centering
  \includegraphics[width=\textwidth]{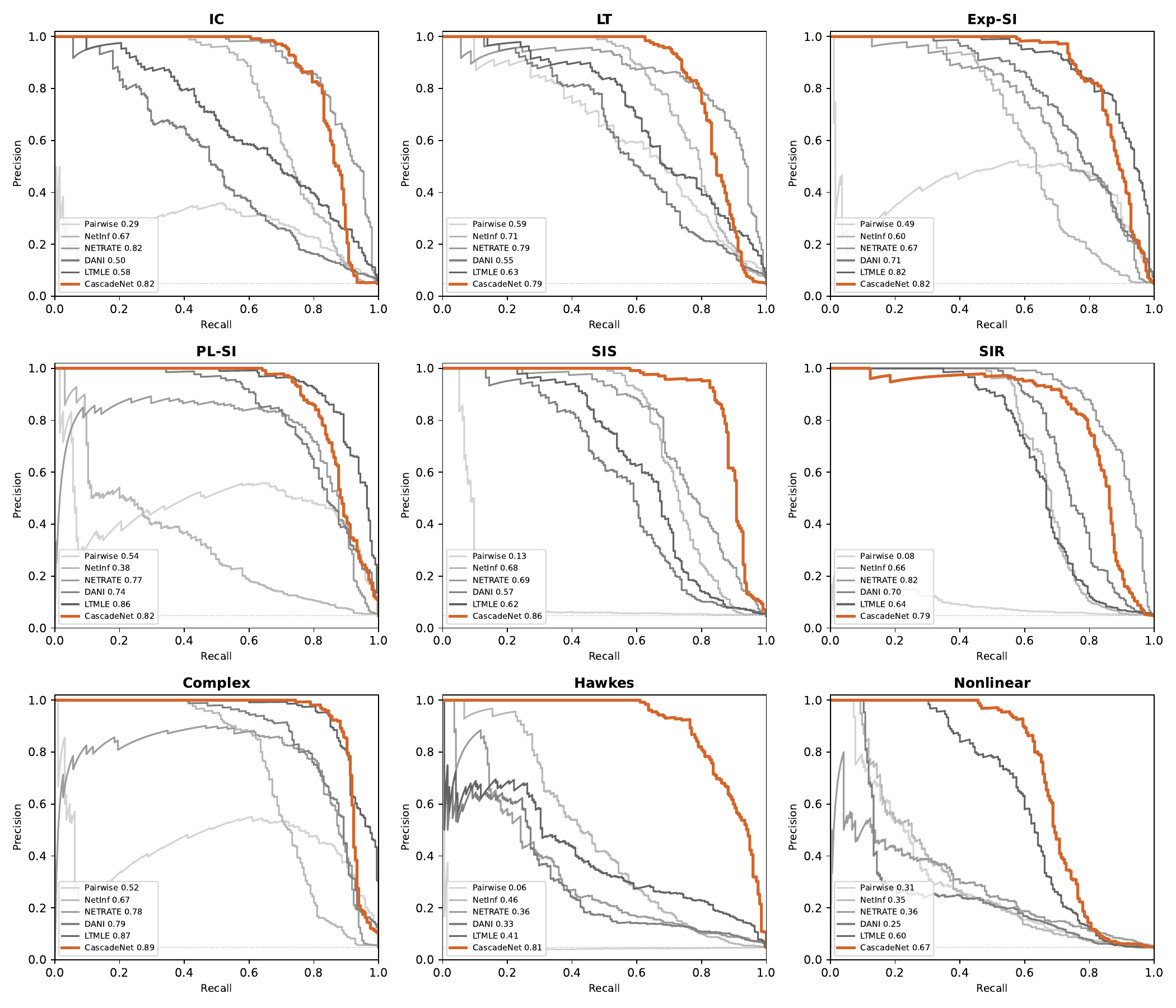}
  \caption{Precision-recall curves for edge detection across nine diffusion
    models ($C = 1{,}000$, ER network).  CascadeNet (orange) dominates all
    baselines, with particularly large margins on SIS, Hawkes, and
    Nonlinear dynamics.  Legend entries show BEP values.}
  \label{fig:pr_curves}
\end{figure}

\appsec{Proofs}{app:proofs}

\subsection{Proof of Neyman Orthogonality (Equation~\ref{eq:score})}
\label{pf:neyman}

The verification below follows the standard Neyman-orthogonality argument for Riesz-representer scores in \citet{chernozhukov2018double} and \citet{chernozhukov2022automatic}; we record it here for completeness, applied to the specific score that targets the network Jacobian.  We verify that the score in \eqref{eq:score} satisfies Neyman orthogonality at the truth for each fixed pair $(i, j)$.  For notational simplicity, write $\theta_0 := \thetaij^0$, $m_0 := m_{0,i}$, $\alpha_0 := \alphaij$, and define
\[
  \psi(W;\, \theta, m, \alpha)
    := \frac{\partial m(Z)}{\partial y_j}
       + \alpha(Z)\bigl(Y_i^+ - m(Z)\bigr) - \theta,
\]
where $W = (Z, Y^+)$ and $Y_i^+$ denotes the $i$-th coordinate of the next-period outcome.  Let $\delta_m$ be an arbitrary square-integrable perturbation of $m_0$, and let $\delta_\alpha$ be an arbitrary square-integrable perturbation of $\alpha_0$.  Consider the paths $m_r = m_0 + r\, \delta_m$ and $\alpha_r = \alpha_0 + r\, \delta_\alpha$.  We show that the Gateaux derivatives of the population moment $\Psi(\theta, m, \alpha) := \E[\psi(W;\, \theta, m, \alpha)]$ with respect to $m$ and $\alpha$, evaluated at $(\theta_0, m_0, \alpha_0)$, are both zero.

\emph{Insensitivity with respect to $m$.}
Differentiate $\Psi(\theta_0, m_r, \alpha_0)$ at $r = 0$:
\[
  \left.\frac{d}{dr} \Psi(\theta_0, m_r, \alpha_0)\right|_{r=0}
    = \E\!\left[\frac{\partial \delta_m(Z)}{\partial y_j} - \alpha_0(Z)\, \delta_m(Z)\right].
\]
By Definition~\ref{def:riesz}, $\E[\alpha_0(Z)\, \delta_m(Z)] = \E[\partial \delta_m(Z)/\partial y_j]$, so the derivative equals zero.

\emph{Insensitivity with respect to $\alpha$.}
Differentiate $\Psi(\theta_0, m_0, \alpha_r)$ at $r = 0$:
\[
  \left.\frac{d}{dr} \Psi(\theta_0, m_0, \alpha_r)\right|_{r=0}
    = \E\!\bigl[\delta_\alpha(Z)\bigl(Y_i^+ - m_0(Z)\bigr)\bigr].
\]
Using iterated expectations and the conditional mean-zero property $\E[Y_i^+ - m_0(Z) \mid Z] = 0$,
\[
  \E\!\bigl[\delta_\alpha(Z)\bigl(Y_i^+ - m_0(Z)\bigr)\bigr]
    = \E\!\bigl[\delta_\alpha(Z)\, \E[Y_i^+ - m_0(Z) \mid Z]\bigr]
    = 0.
\]
Therefore the score is Neyman-orthogonal at the truth.
\hfill$\square$

\subsection{Proof of Theorem~\ref{thm:rate_JK}}
\label{pf:thm_rate_JK}

The argument follows the standard cross-fitting / orthogonal-score template of \citet{chernozhukov2018double} and \citet{chernozhukov2022automatic}; the only modification relative to those references is that the unit of cross-fitting is the cascade trajectory rather than an i.i.d.\ observation, which is required because within-trajectory observations are dependent.  We prove asymptotic normality of the cross-fitted debiased estimator for a fixed pair $(i, j)$.  Write $\theta_0 := \thetaij^0$, $\alpha_0 := \alphaij$, and $m_0 := m_{0,i}$.  For each trajectory $c$, define the oracle trajectory-level score
\[
  \phi_0(W_c)
    := \frac{1}{T_c - 1} \sum_{t=1}^{T_c - 1}
       \psiij\!\bigl(Z_{c,t}, Y_{c,t+1};\, \theta_0, m_0, \alpha_0\bigr),
\]
where $\psiij$ is the score in \eqref{eq:score}.  By Assumption~\ref{ass:variance}, $\E[\phi_0(W_c)] = 0$ and $\E[\phi_0(W_c)^2] = V_{ij} \in (0, \infty)$.

Let $I_1, \ldots, I_K$ denote the cross-fitting folds.  For each $c \in I_k$, define the estimated trajectory-level signal
\[
  \hat{\phi}_c
    := \frac{1}{T_c - 1} \sum_{t=1}^{T_c - 1}
       \left[
         \frac{\partial \hat{m}_i^{(-k)}(Z_{c,t})}{\partial y_j}
         + \hat{\alpha}_{ij}^{(-k)}(Z_{c,t})\bigl(Y_{c,t+1,i} - \hat{m}_i^{(-k)}(Z_{c,t})\bigr)
         - \theta_0
       \right].
\]
Then $\hat{\theta}_{ij}^{\mathrm{deb}} - \theta_0 = n^{-1} \sum_{c=1}^n \hat{\phi}_c$, and we will show
\[
  \sqrt{n}\,(\hat{\theta}_{ij}^{\mathrm{deb}} - \theta_0)
    = \frac{1}{\sqrt{n}} \sum_{c=1}^n \phi_0(W_c) + o_p(1),
\]
from which the result follows by a central limit theorem.

\paragraph{Step 1: Decomposition.}
Write
\begin{equation}
  \label{eq:proof_decomp}
  \sqrt{n}\,(\hat{\theta}_{ij}^{\mathrm{deb}} - \theta_0)
    = \frac{1}{\sqrt{n}} \sum_{c=1}^n \phi_0(W_c)
      + \frac{1}{\sqrt{n}} \sum_{c=1}^n \bigl(\hat{\phi}_c - \phi_0(W_c)\bigr).
\end{equation}
The first term is the oracle empirical process; the second is the remainder from estimating the nuisances.

\paragraph{Step 2: CLT for the oracle term.}
By Assumption~\ref{ass:iid}, the trajectories $W_1, \ldots, W_n$ are i.i.d.  By Assumption~\ref{ass:variance}, $\E[\phi_0(W_c)] = 0$ and $\E[\phi_0(W_c)^2] = V_{ij} < \infty$.  The standard central limit theorem gives
\begin{equation}
  \label{eq:proof_clt}
  \frac{1}{\sqrt{n}} \sum_{c=1}^n \phi_0(W_c)
    \xrightarrow{d} \mathcal{N}(0, V_{ij}).
\end{equation}

\paragraph{Step 3: Expansion of the remainder.}
Fix a fold $k$ and condition on the training sample used to estimate $\hat{m}^{(-k)}$ and $\hat{\alpha}_{ij}^{(-k)}$.  For each held-out trajectory $c \in I_k$,
\[
  \hat{\phi}_c - \phi_0(W_c) = A_c + B_c + C_c + D_c,
\]
where
\begin{align*}
  A_c &:= \frac{1}{T_c - 1} \sum_{t=1}^{T_c - 1}
            \left[\frac{\partial \hat{m}_i^{(-k)}(Z_{c,t})}{\partial y_j}
                  - \frac{\partial m_0(Z_{c,t})}{\partial y_j}\right],
  \\
  B_c &:= \frac{1}{T_c - 1} \sum_{t=1}^{T_c - 1}
            \bigl(\hat{\alpha}_{ij}^{(-k)}(Z_{c,t}) - \alpha_0(Z_{c,t})\bigr)
            \bigl(Y_{c,t+1,i} - m_0(Z_{c,t})\bigr),
  \\
  C_c &:= -\frac{1}{T_c - 1} \sum_{t=1}^{T_c - 1}
            \alpha_0(Z_{c,t})\bigl(\hat{m}_i^{(-k)}(Z_{c,t}) - m_0(Z_{c,t})\bigr),
  \\
  D_c &:= -\frac{1}{T_c - 1} \sum_{t=1}^{T_c - 1}
            \bigl(\hat{\alpha}_{ij}^{(-k)}(Z_{c,t}) - \alpha_0(Z_{c,t})\bigr)
            \bigl(\hat{m}_i^{(-k)}(Z_{c,t}) - m_0(Z_{c,t})\bigr).
\end{align*}

\paragraph{Step 4: Orthogonality cancels the first-order terms.}
Conditional on the training sample, the held-out trajectories are independent of the nuisance estimators by cross-fitting.  For the terms $A_c$ and $C_c$, the Riesz property gives
\[
  \E[A_c + C_c \mid \text{training sample}]
    = \E\!\left[
        \frac{\partial (\hat{m}_i^{(-k)} - m_0)(Z)}{\partial y_j}
        - \alpha_0(Z)\bigl(\hat{m}_i^{(-k)} - m_0\bigr)(Z)
        \;\middle|\; \text{training sample}
      \right]
    = 0.
\]
For the term $B_c$, using $\E[Y_{c,t+1,i} - m_0(Z_{c,t}) \mid Z_{c,t}] = 0$,
\[
  \E[B_c \mid \text{training sample}] = 0.
\]
Thus the first-order nuisance effects cancel in expectation.

\paragraph{Step 5: Control of the second-order term.}
The remaining term $D_c$ is second order.  By Cauchy--Schwarz,
\[
  \bigl|\E[D_c \mid \text{training sample}]\bigr|
    \le \|\hat{\alpha}_{ij}^{(-k)} - \alpha_0\|_{L^2(P_Z)}
        \,\|\hat{m}_i^{(-k)} - m_0\|_{L^2(P_Z)}.
\]
By Assumption~\ref{ass:rates}, each factor is $o_p(n^{-1/4})$, so their product is $o_p(n^{-1/2})$, and hence $\E[D_c \mid \text{training sample}] = o_p(n^{-1/2})$.  Because trajectory lengths are uniformly bounded by Assumption~\ref{ass:iid}, averaging over trajectories preserves this rate, so $n^{-1/2} \sum_{c=1}^n D_c = o_p(1)$.  Similarly, the centered fluctuations of $A_c$, $B_c$, and $C_c$ are negligible under Assumptions~\ref{ass:smooth} and~\ref{ass:rates}, since the derivative estimation error converges to zero in $L^2$, the nuisance estimators are cross-fitted, and trajectory lengths are uniformly bounded.  Hence $n^{-1/2} \sum_{c=1}^n (A_c + B_c + C_c) = o_p(1)$.  Combining the previous two displays yields
\begin{equation}
  \label{eq:proof_remainder}
  \frac{1}{\sqrt{n}} \sum_{c=1}^n \bigl(\hat{\phi}_c - \phi_0(W_c)\bigr) = o_p(1).
\end{equation}

\paragraph{Step 6: Conclusion.}
Substituting \eqref{eq:proof_clt} and \eqref{eq:proof_remainder} into \eqref{eq:proof_decomp},
\[
  \sqrt{n}\,(\hat{\theta}_{ij}^{\mathrm{deb}} - \theta_0)
    \xrightarrow{d} \mathcal{N}(0, V_{ij}),
  \qquad
  V_{ij} = \E[\phi_0(W_c)^2].
\]
This proves Theorem~\ref{thm:rate_JK}.
\hfill$\square$

\subsection{Graph Neural Network Model and Sparse Gradient}
\label{pf:gnn_grad}

When the transition function involves complex nonlinear interactions
that a linear index cannot capture, CascadeNet can be implemented as
a Graph Neural Network using a GraphSAGE-style architecture
\citep{hamilton2017inductive}:
\begin{align}
  \mathbf{h}_i &= \sigma_{\mathrm{enc}}\!\bigl(W_{\mathrm{enc}}\,[y_{t,i},\, X_{t,i}]\bigr)
    \in \R^H,
    \label{eq:gnn_h}
  \\
  \mathrm{agg}_i &= A\, \mathbf{h}
    \qquad [A = \mathrm{softmax}(QK\trans/\sqrt{H})],
    \label{eq:gnn_agg}
  \\
  \hat{m}_i &= W_{\mathrm{dec}}\,[\mathbf{h}_i,\, \mathrm{agg}_i].
    \label{eq:gnn_dec}
\end{align}
Here $\mathbf{h}_i$ is the node embedding for agent $i$, encoding its
own state and covariates.  The attention matrix
$A = \mathrm{softmax}(QK\trans/\sqrt{H})$ captures how agent $i$
aggregates information from its neighbors, with $Q$ and $K$ as
learnable query and key matrices.  The decoder combines the node
embedding and the aggregated neighbor information to produce the
predicted transition $\hat{m}_i$.

With a known adjacency matrix $W$ (no self-loops, $W_{ii} = 0$),
the embedding gradient has a sparse structure:
\begin{equation}
  \label{eq:sparse_grad}
  \frac{\partial \embed_i}{\partial y_j} =
  \begin{cases}
    \bigl[\nabla_{y_i}\mathbf{h}_i,\; \mathbf{0}\bigr] & \text{if } j = i, \\[4pt]
    \bigl[\mathbf{0},\; W_{ij}\,\nabla_{y_j}\mathbf{h}_j\bigr] & \text{if } j \ne i.
  \end{cases}
\end{equation}

\emph{Proof.}
Decompose $\embed_i = [\mathbf{h}_i, \mathrm{agg}_i]$.
By \eqref{eq:gnn_h}, $\mathbf{h}_i$ depends only on $(y_i, X_i)$, so
$\partial \mathbf{h}_i / \partial y_j = \mathbf{0}$ for $j \ne i$.
By \eqref{eq:gnn_agg} with $W_{ii} = 0$,
$\partial \mathrm{agg}_i / \partial y_i = W_{ii}\,\nabla_{y_i}\mathbf{h}_i
= \mathbf{0}$ and
$\partial \mathrm{agg}_i / \partial y_j = W_{ij}\,\nabla_{y_j}\mathbf{h}_j$
for $j \ne i$.  Concatenating gives \eqref{eq:sparse_grad}.
\hfill$\square$

\subsection{Extension to Non-Markovian Dynamics}
\label{app:nonmarkov}

The Markovian assumption in \eqref{eq:dgp} requires that $Y_{c,t-1}$
affects $Y_{c,t+1}$ only through $Y_{c,t}$.  When this assumption is
violated, the framework can be extended by augmenting the state vector
to include lagged values.  Specifically, define the augmented state
$\tilde{Y}_{c,t} = (Y_{c,t}, Y_{c,t-1}, \ldots, Y_{c,t-L})$ for
some lag order $L \geq 1$.  The transition function then becomes
$\tilde{Y}_{c,t+1} = \tilde{m}_0(\tilde{Y}_{c,t}, X_c; \theta)$,
which is Markovian in the augmented state.  All results in the paper
carry through with $Y$ replaced by $\tilde{Y}$, at the cost of
increasing the dimensionality of the state from $N$ to $N(L+1)$.
In practice, one or two lags typically suffice for diffusion
processes with short memory.

\subsection{Supplementary Simulation: Debiasing Validation}
\label{app:debiasing_sim}

To directly validate the Riesz debiasing correction with analytical
ground truth, we present a simulation using a tanh DGP where the true
Jacobian admits a closed form.

\paragraph{Data-generating process.}
We simulate $N = 20$ nodes, $T = 10$ time steps, $R = 800$ trajectories:
\begin{equation}
  y_{t+1,i} = \tanh\!\bigl(\alpha\, y_{t,i}
    + \beta\,(Wy_t)_i + \gamma\, x_i\bigr) + \varepsilon_{t+1,i},
  \qquad \varepsilon_{t+1,i} \sim \mathcal{N}(0,\sigma^2),
\end{equation}
with $\alpha = 0.3$, $\beta = 0.5$, $\gamma = 1.0$, $\sigma = 0.1$.

\paragraph{Results.}
Table~\ref{tab:app:sim} reports results over 1{,}000 evaluation points.
The naive Jacobian is nearly uninformative ($r = 0.070$) due to
$\ell_2$ attenuation.  The Riesz correction recovers the signal
($r \to 0.768$), reducing Frobenius error by $1.70\times$.

\begin{table}[ht]
\centering
\caption{Debiasing Validation on Tanh DGP: Naive vs.\ Debiased Jacobian Estimates.}
\label{tab:app:sim}
\begin{tabular}{lcc}
\toprule
 & Naive & Debiased \\
\midrule
Mean Frobenius error  & 1.255 & \textbf{0.738} \\
Correlation $r$       & 0.070 & \textbf{0.768} \\
OLS $R^2$             & 0.005 & \textbf{0.589} \\
Error reduction       & \multicolumn{2}{c}{$\mathbf{1.70\times}$} \\
\bottomrule
\end{tabular}
\end{table}

\subsection{Robustness to Model Misspecification}
\label{app:misspec}

In the main text, the asymptotic results are derived under the
assumption that the transition model $m(y, x;\, \beta)$ is correctly
specified, so that there exists a parameter vector $\beta_0$ satisfying
\[
  \E[Y_{t+1} \mid Z_t] = m(Z_t;\, \beta_0),
  \qquad Z_t = (Y_t, X_t).
\]
In practice, the transition model may be misspecified.  In this
appendix we briefly discuss how the results extend to this case.

\paragraph{Pseudo-true parameter.}
When the model is misspecified, the estimator typically converges not
to a structural parameter but to a pseudo-true parameter
$\beta^\star$.  For example, under squared prediction loss,
\[
  \beta^\star
    \;\in\;
    \arg\min_{\beta}\;
      \E\!\bigl[\|Y_{t+1} - m(Z_t;\, \beta)\|^2\bigr].
\]
Equivalently, $\beta^\star$ satisfies the population score condition
\[
  \E\!\bigl[g(Z_t;\, \beta^\star)\,
    \bigl(Y_{t+1} - m(Z_t;\, \beta^\star)\bigr)\bigr] = 0,
\]
where $g(Z_t;\, \beta) := \partial_\beta m(Z_t;\, \beta)$.
Under standard regularity conditions, the debiased estimator satisfies
\[
  \sqrt{n}\,(\hat{\beta}^{\mathrm{deb}} - \beta^\star)
    \;\xrightarrow{d}\;
    \mathcal{N}(0,\, \Sigma_{\beta}^\star),
\]
with sandwich covariance
$\Sigma_{\beta}^\star = A_\star^{-1}\, \Omega_\star\, A_\star^{-1\top}$.
The debiased Jacobian estimator therefore continues to be
asymptotically normal under misspecification, centered at the
Jacobian implied by the best approximation $\beta^\star$ within the
model class.  This approximation bias is expected to be small when
the model class $\mathcal{M}$ is sufficiently flexible to closely
approximate the true conditional expectation.

\end{document}